%% file: main.tex
\documentclass[10.2pt, twocolumn]{article}
\usepackage[round]{natbib}
\usepackage{graphicx}
\usepackage[utf8]{inputenc}
\usepackage{color}
\usepackage{soul}
\usepackage{booktabs}
\usepackage{placeins}
\usepackage{enumitem}
\usepackage[utf8]{inputenc}
\usepackage{fancyhdr}
\usepackage{hyperref}

\newcommand\independent{\protect\mathpalette{\protect\independenT}{\perp}}
\def\independenT#1#2{\mathrel{\rlap{$#1#2$}\mkern2mu{#1#2}}}

\usepackage{enumitem}

\usepackage{longtable}
\usepackage{array}
\usepackage{geometry}
\usepackage[american]{babel}

\usepackage{amsmath}
\usepackage{amssymb}
\usepackage{mathtools}
\usepackage{amsthm}

\usepackage{xr}
\makeatletter

\newcommand*{\addFileDependency}[1]{
\typeout{(#1)}
\IfFileExists{#1}{}{\typeout{No file #1.}}
}\makeatother

\newcommand*{\myexternaldocument}[1]{%
\externaldocument[app:]{#1}%
\addFileDependency{#1.tex}%
\addFileDependency{#1.aux}%
}
\myexternaldocument{supplement}

\usepackage{comment}
\usepackage{makecell}
\usepackage{booktabs} 

\usepackage{dsfont}

\usepackage{xcolor}
\usepackage{soul}

\usepackage{algorithm}
\usepackage{algpseudocode}

\input{header}

\title{MINTY: Rule-based Models that Minimize the Need for Imputing Features with Missing Values}
 \author{Lena Stempfle $\And$ Fredrik D. Johansson}
\date{Department of Computer Science and Engineering (CSE), Chalmers University of Technology, Sweden\\ \href{mailto:stempfle@chalmers.se}{stempfle@chalmers.se} \href{mailto:fredrik.johansson@chalmers.se}
{fredrik.johansson@chalmers.se}}

\begin{document}
\maketitle

\begin{abstract}
Rule models are often preferred in prediction tasks with tabular inputs as they can be easily interpreted using natural language and provide predictive performance on par with more complex models. However, most rule models’ predictions are undefined or ambiguous when some inputs are missing, forcing users to rely on statistical imputation models or heuristics like zero imputation, undermining the interpretability of the models. In this work, we propose fitting concise yet precise rule models that learn to avoid relying on features with missing values and, therefore, limit their reliance on imputation at test time. We develop \MINTY{}, a method that learns rules in the form of disjunctions between variables that act as replacements for each other when one or more is missing. This results in a sparse linear rule model, regularized to have small dependence on features with missing values, that allows a trade-off between goodness of fit, interpretability, and robustness to missing values at test time. We demonstrate the value of \MINTY{} in experiments using synthetic and real-world data sets and find its predictive performance comparable or favorable to baselines, with smaller reliance on features with missing values. 
\end{abstract}

\section{Introduction}
Linear rule models find extensive use in prediction tasks such as classification, regression, and risk scoring~\citep{furnkranz2012foundations, wei2019generalized, margot2021new}, and are particularly favored in domains where interpretability holds paramount importance. In the same domains, it is common for some of the variables used in the learned rules to be unobserved, missing at the time of prediction.

Established approaches to prediction with incomplete data at test time, include Bayesian modeling~\citep{webb2010naive}, fallback default rules~\citep{twala2008good,chen2016xgboost}, weighted estimating equations~\citep{ibrahim2005missing}, prediction with missingness indicators~\citep{le2020neumiss} and imputation~\citep{rubin1976inference}. Although imputation is powerful, it is not always optimal under test-time missingness~\citep{lemorvan20a_linear} and often assumes that data is missing at random
(MAR)~\citep{carpenter2012multiple, seaman2013meant}. 

A limitation of existing methods is that they either i) are specific to less interpretable model classes or ii) undermine the interpretability offered by rule-based models by relying on less interpretable auxiliary models (for imputation, estimation weighting)~\citep{rubin1988overview} or on parameters associated with missingness itself (fallback rules, missingness indicators)~\citep{jones1996, chen2016xgboost, stempfle2023sharing}.

To address these shortcomings, we aim to \textit{learn interpretable rule models that inherently limit the need for imputation of features with missing values}. We call our solution \MINTY{}, which handles  \textit{\textbf{mi}ssingness} and provides {\textit{\textbf{int}erpretablit\textbf{y}}} by learning generalized linear rule  models (GLRM) where literals of single variables are grouped in disjunctive rules so that the truth value of a rule can be determined when \emph{one} of the literals is observed and true, no matter if the others are missing. This idea exploits redundancy in the covariate set inherent to many prediction tasks by allowing observed variables to be used as \emph{replacements} for missing ones. Through a tunable regularization penalty on rules whose value can frequently not be determined, we mitigate the reliance on imputation at test time.

\paragraph{Illustrative Example: Alzheimer's Progression.}
Figure~\ref{fig:example} illustrates a disjunctive linear rule model for predicting cognitive decline. In the model, rules (left) are combined with coefficients (right) to calculate a predicted change in cognitive function (measured by ADAS13). The example shows the model's prediction for a patient, Anna, whose observed variables are displayed at the bottom. If at least one literal in each rule is observed and true, the added score is the same whether other variables in the rule are missing. For Anna, her TAU protein fragment level is observed to be in the range (\texttt{Tau $\leq$ 191}), while a measurement for PTAU is missing. Despite this, the second rule can be evaluated and is true, contributing -5.2 to the final score. 
Similarly, the first rule is true, as we know that for Anna, \texttt{MMSE = 24}, even though she has not received a prior AD diagnosis. In the case of a single-feature rule with a missing value, (e.g., \texttt{Married=True}), we default to zero-imputation, and no score is added to the total. This is common practice in the use of risk scores~\citep{afessa2005influence}, but may be possible to avoid by learning disjunctive rules whose value can be determined by a single observed feature and have to be zero-imputed less often.

\begin{figure}[t]
\centering\includegraphics[width=1.\columnwidth]{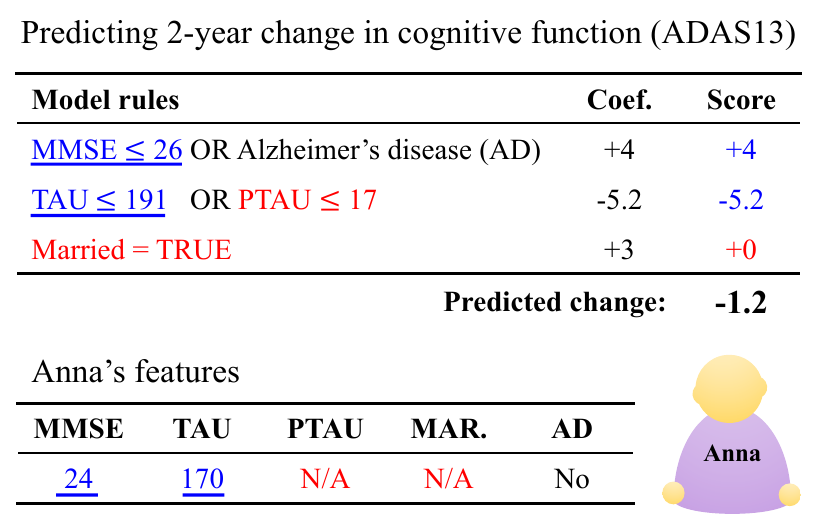}%
\caption{Illustrative example of scoring system predicting cognitive decline, measured by a change in the ADAS13 cognitive function score, using the \ADNI{} data including incomplete data. The blue, underlined features indicate that these variables are observed for the specific patient, Anna, and the red shows that the observations for the variables are missing.\label{fig:example}}
\end{figure}

\paragraph{Contributions.}
Our contributions can be summarized as follows: 
1) We propose \MINTY{}, a generalized linear rule model, which uses disjunctive rules to exploit redundancy in the input variables, mitigating the need for imputation. 
2) We optimize \MINTY{} by adapting the column generation strategy of~\citet{wei2019generalized}, iteratively adding rules to the model based on a tunable trade-off between high predictive performance and small reliance on missing values. 
3) We perform empirical experiments comparing \MINTY{} to baselines that either handle missing values natively or rely on imputation. The results show that our proposed method achieves comparable prediction performance to larger black-box models and models that rely much more on features with missing values in prediction.
%
%
\section{Rule Models \& Features with Missing Values}
\label{sec:problem}
%
We consider predicting an outcome $Y \in \mathbb{R}$ based on a vector of $d$ input features $X = [X_1, ..., X_d]^\top \in \bbR^d$ when the value of any feature $X_j$ may be missing \emph{at training time or test time}. Missingness is determined by a random binary mask $M  = [M_1, ..., M_d]^\top \in \{0,1\}^d$ applied to a complete variable set $X^*$, such that $X_j = X^*_j$ if $M_j = 0$, and $X_j = \na$ if $M_j = 1$.

Our goal is to minimize the expected error in prediction, $R(h) \coloneqq \E_p[L(h(X), Y)]$, over a distribution $p$, using a hypothesis $h$ that handles missing values in the input $X$. $L$ is a loss function such as the squared error or logistic loss. 
To learn, we are given a training set of examples $D = \{(x_i, m_i, y_i)\}_{i=1}^m$, assumed to be drawn i.i.d. from $p$. Here, $x_i=[x_{i1}, ...x_{id}]^\top$ is the (partially missing) feature vector of sample $i$, and $m_i, y_i$ defined analogously. We let $\bX \in (\{0,1\} \cup \{\na\})^{n\times d}, \bM \in \{0,1\}^{n\times d}, \bY \in \mathbb{R}^{n\times 1}$ denote feature matrices, missingness masks and outcomes for all observations in $D$.

We say that a hypothesis $h$ \emph{relies on features with missing values} for an observation $x_i$ if there is a feature $j$ such that 1) $x_{ij} = \na$, and 2) computing $h(x_i)$ requires evaluating $x_{ij}$. We use a binary indicator $\rho_h(x_i) \in \{0,1\}$ to indicate reliance.
For example, a linear model used with imputation (e.g., zero imputation or MICE) relies on features with missing values whenever its input $x_i$ has any missing value. An XGBoost ensemble $h$ has $\rho_h(x_i)=1$ if $x_i$ passes a ``default'' rule in its traversal through any of the model's trees. If the tree contains default rules, but $x_i$ traverses neither of them, $\rho_h(x_i)=0$. We denote the average reliance $\bar{\rho}(h) = \E_{X\sim p}[\rho(X)]$.

We propose \MINTY{}, a learning algorithm that mitigates reliance on features with missing values by making predictions using \emph{disjunctions} (or-clauses) of literals, e.g., ``(Age $>$ 60) or (Prior stroke)''. If the value of ``Age'' is missing, but ``Prior stroke'' is True, the rule no longer depends on the value of ``Age''. This creates robustness by redundancy. Moreover, \MINTY{} adds regularization to ensure that its rules can be evaluated with high probability despite missing values. We build our method on generalized linear rule models. 

\subsection{Generalized Linear Rule Models}
In rule learning, features represent binary logical literals, where $X_{ij} = 1$ means that literal $j$ is True for observation $i$. For instance, feature $j$ may represent the literal $Age \geq 70$, and a subject $i$ that is 73 years old would have $x_{ij} = 1$. There are standard ways to transform continuous and categorical values to literals, such as discretization by quantiles and dichotomization~\citep{rucker2015researcher}.

\citet{wei2019generalized} defined generalized linear rule models (GLRM) using three components: 
\begin{enumerate}
\item \emph{Rule definitions} $z_{k} = [z_{1k}, ..., z_{dk}]^\top \in \{0,1\}^{d}$, for rules $k=1, ..., K$, which define logical rules in terms of inclusion indicators $z_{jk}$ of literals $j\in [d]$.
\item \textit{Rule activations} $a_{i} = [a_{i1}, ..., a_{iK}]^\top \in \{0,1\}^K$, where $a_{ik}$ indicates whether rule $k$ is satisfied by observation $x_i$.
\item \textit{Rule coefficients}, $\beta = [\beta_1, ..., \beta_K]^\top \in \mathbb{R}^K$, where $\beta_k$ relates  rule $k$ to the predicted outcome. Letting rule 1 always be true, $\beta_1$ is the intercept.
\end{enumerate}

In this work, we use only \emph{disjunctive} GLRMs, were the activation of rule $k$ for complete $x_i$ is defined as
$$
a_{ik} \coloneqq \bigvee_{j=1}^d x_{ij} z_{jk} = \max_{j \in [d]}  x_{ij} z_{jk}~.
$$
In other words, $a_{ik} = 1$ if for \emph{any} feature $j$, the literal is True ($x_{ij}=1$) and $j$ is included in rule $k$ ($z_{jk}=1$).

A GLRM predicts the outcome $y_i$ for a \emph{complete} input $x_i$ as a generalized linear model of the rule indicators,
$$
\hat{y}_i = \Phi'(\eta_i) \; \mbox{ where } \; \eta_i = a_i^\top \beta
$$
where $\Phi$ is the log-partition function of the conditional distribution for an exponential family model $p(Y=y \mid X=x) = h(y)\exp(\eta y - \Phi(\eta))$. For linear regression, $\Phi'(\eta) = \eta$ and for logistic regression $\Phi'(\eta)  = 1/(1+\exp(-\eta))$ is the logistic function $\sigma(\eta)$.

\subsection{Mitigating Reliance on Missing Features with Disjunctive Rules}
GLRMs are not designed to handle missing values by default. In this work, we treat the truth value of rules as potentially missing as well, depending on the literals included in the rule. Concretely, 
$$
a_{ik} = \left\{
  \begin{array}{@{}ll@{}}
    1, & \exists j \in z_k : m_{ij}=0 \land x_{ij}=1 \\ 
      0, & \forall j \in z_k : m_{ij}=0 \land x_{ij}=0 \\
      \na, & \forall j \in z_k : m_{ij}=1 \lor x_{ij}=0 \\ 
  \end{array}\right.~.
$$
where $(j \in z_k) \Leftrightarrow (z_{jk} = 1)$.  For example, 
$$
  (x_1 \lor x_2) = \left\{
  \begin{array}{@{}ll@{}}
    1, & x_1 = 1 \;\text{or}\; x_2 = 1 \\ 
      0, & x_1 = 0 \;\text{ and }\; x_2 = 0 \\
      \na, & (x_1=0 \;\text{ and }\; x_2=\na) \; \text{ or }\\ 
      & (x_1=\na \;\text{ and }\; x_2=0)
  \end{array}\right.~.
$$
To predict using a rule $k$ such that $a_{ik}=\na$, we would still need to impute some of the missing literals. 

On the other hand, \emph{evaluating the disjunction does not rely on all of its literals being observed}. As long as one literal is observed and True, we know that the value of the disjunction is True as well. Hence, the reliance $\bar{\rho}(h)$ for a disjunctive GLRM $h$ can be lower than for, e.g., a linear model applied to the same features.

%
%
\section{MINTY: Rule Models that avoid Imputation of Missing Values}
\label{sec:methodoloy}
We aim to learn a small set of rules $\cS$ and coefficients $\beta$ that minimize the regularized empirical risk, with a small expected reliance on features with missing values. Let $\cK$ denote an index over \emph{all possible disjunctions} of $d$ binary features and let $\cS \subseteq \cK$ be the subset of rules used by our model, such that $k$ defines $z_k$ and thus $a_{ik}$ for all observations $i$. Then, let $\rho_{ik} = \mathds{1}[a_{ik} = \na]$ indicate the reliance of rule $k$ on missing values in observation $x_i$. 
With a parameter $\gamma \geq 0$ used to control the average reliance on missing features $\bar{\rho}_k$ for included rules, and a sparsity penalty $\lambda_k > 0$ for $k \in \cK$, we aim to solve,
\begin{equation}\label{eq:objective}
\min_{\beta, \cS} \frac{1}{n} \sum_{i=1}^n \left[(\beta^\top a_{i\cS} - y_i)^2 + \sum_{k \in \cS} (\gamma\rho_{ik} + \lambda_k)|\beta_k| \right]
\end{equation}

Following~\citet{wei2019generalized}, we use an $\ell_1$-penalty for controlling the size of the rule model, with parameter $\lambda_k = \lambda_0 + \lambda_1\|z_k\|_1$. The latter term counts the number of literals in disjunction $k$. By choosing $\lambda_0, \lambda_1$, we can control the number and size of rules used by the model.

If we let $\cS$ be the set of all possible disjunctions $\cK = \{0,1\}^d$, our learning problem reduces to a LASSO-like problem with active rules determined by the sparsity pattern in $\beta$, but with a number of rules and coefficients that grows exponentially with $d$. Even for moderate-size problems, these would be intractable to enumerate. Instead, we follow the column-generation strategy by~\citet{wei2019generalized}, which searches the space of disjunctions and builds up $S$ incrementally. 

The idea is to first solve problem~\eqref{eq:objective} restricted to a small set of candidate rules $\hcS = \cS_0$, in our case just the intercept rule. 
Given a current set of disjunctions $\hcS$ and estimated coefficients $\hbeta$, a new rule is added by finding the disjunction that aligns the most with the residual of the current model, $\bR = \bA_{\hcS}\hbeta - \bY$, where $\bA_{\hcS} = [a_{1\cdot}, \ldots, a_{n\cdot}]^\top$ is the matrix of rule assignments for all observations in the training set w.r.t. $\hcS$. 

This procedure is justified by the optimality conditions of \eqref{eq:objective} which imply that at an optimal solution, the partial derivative with respect to both the positive and negative components of $\beta$ must be non-negative. Optimality can therefore be determined by minimizing $\pm \frac{1}{n} \mathbf{R}^\top \ba + \mathcal{R}(a)$ over the corresponding activations of a new rule $\ba \in \{0,1\}^n$ (with $\mathcal{R}(a)$ corresponding to regularization terms, specified further below). 

To avoid computation with \na{} values, we zero-impute $X$, defining $\zx_{ij} = \mathds{1}[m_{ij}=0]x_{ij}$, keeping track of missing values in the mask $M$. In principle, other imputation could be used. 
We choose the next rule as defined by the minimizer $z^*$ of the following two problems ($\pm$),
\begin{equation}\label{eq:relaxed_minty}
\begin{aligned}
& \underset{\substack{ z\in \{0,1\}^d \\ a, \rho \in \{0,1\}^n }}{\text{minimize}}
\;\;\;\; \pm \frac{1}{n} \sum_{i=1}^n (r_i  a_i + \gamma \rho_i) + \lambda_0 + \lambda_1  \sum_{j=1}^d z_{j}  \\
 & \text{subject to }
 \;\;\; a_{i} = \sum_{k=1}^K \max(\zx_{ij}z_{j}) \\
 & \forall{i} : \rho_i = \underbrace{(1 - \max_j [(1-M_{ij})z_{j}\zx_{ij}])}_{(i)}\underbrace{(\max_j M_{ij}z_j)}_{(ii)}%
\end{aligned}%
\end{equation}%
We let $z_{k^*}, a_{k^*}, \rho_{k^*}$ refer to the optimizers of \eqref{eq:relaxed_minty}, for the sign with smallest objective value, and $\delta_{k^*}$ to the corresponding objective. The first constraint in~\eqref{eq:relaxed_minty} makes sure that rule activations $a_{i}$ correspond to a disjunction of literals $\zx_{ij}$ as indicated by $z$.  The constraint on $\rho_i$ ensures that reliance on missing factors is counted only when (i) there is no observed True literal in the rule, and (ii) at least one literal is missing.

\begin{algorithm}[tb]
    \caption{\MINTY{} learning algorithm}
    \label{alg:algorithm}
    \textbf{Input}: $\bX, \bM \in \{0,1\}^{n \times d}$, $\bY \in \bbR^n$ \\
    \textbf{Parameters}: $\lambda_0, \lambda_1, \gamma \geq 0, k_{max}\geq 1$ \\
    \textbf{Output}: $\cS$, $\beta$ 
\begin{algorithmic}[1] 
        \State Initialize $\hcS = \{0\}$ where $0$ is the intercept rule
        \State Initialize $\delta_{k*} = -\infty$
        \State Let $\zbX$ be zero-imputed $\bX$, $\zx_{ij} = \mathds{1}[m_{ij}=0]x_{ij}$
        \State Let $l = 0$
        \While{$\delta_{k*} < 0 $, $l < k_{max}$}
        \State $\beta \gets \argmin_\beta \cO(\zbX, \bY, \hcS, \lambda_0, \lambda_1, \gamma)$ \Comment{\eqref{eq:objective}}
        \State $a_{ik} = \max_{j\in [d]}z_{jk}\zx_{ij}$ for $i \in [n], k \in \hcS$
        \State $r_i = \sum_{k \in \hcS}\beta_ka_{ik} - y_i$ for $i \in [n]$
        \State $z_{k^*}, \delta_{k^*} \gets$ ADD$(\zbX, \bY, \bR, \lambda_0, \lambda_1, \gamma)$  \Comment{\eqref{eq:relaxed_minty}}
        \If{$\delta_{k^*} \geq 0$:} 
        \State \textbf{break}. The current solution is optimal.
        \Else
        \State Append new rule $k^*$ to $\hcS$,  
        \State $l \gets l+1$
        \EndIf
        \EndWhile
        \State $\hbeta \gets \argmin_\beta \cO(\zbX, \bY, \hcS, \lambda_0, \lambda_1, \gamma)$
        \State \textbf{return} $\hcS, \hbeta$ 
    \end{algorithmic}
\end{algorithm}

When no rule can be found with a negative solution to \eqref{eq:relaxed_minty}, or a maximum number of rules $k_{max}$ has been reached, the algorithm terminates. We finish by solving \eqref{eq:objective} with respect to $\beta$ for fixed $\hat{\mathcal{S}}$. The algorithm can be adapted to generalized linear models like logistic regression, without changing the rule generation procedure, as shown by~\citet{wei2019generalized}. 
We summarize our method, referred to as \MINTY{}, in Algorithm~\ref{alg:algorithm}.

%
%
\subsection{Solving the Rule Generation Problem}
\label{sec:minty_opt}
The problem in \eqref{eq:relaxed_minty} is an integer linear program with nonlinear constraints. We consider two methods in experiments: Exact solutions using the off-the-shelf optimization toolkit Gurobi~\citep{gurobi}, and approximate solutions using a heuristic beam search algorithm, as used by~\citet{oberst2020characterization}. 

For the beam search algorithm, we initialize the beam to contain all disjunctions of a single literal. We then retain the top-$W_{b}$ of these, in terms of the objective in~\eqref{eq:relaxed_minty}. Then, we generate the next set of candidates by adding one literal to all disjunctions, and evaluate these in the same way, retaining the top-$W_{b}$ and proceeding in the same way until at most $D_{b}$ literals have been added. Throughout, we keep track of the rule with the smallest objective, no matter its size, and return this once the beam has reached its maximum depth. In experiments, we let the beam width be $W_b=d$ (the number of features) and depth $D_b=7$. The time complexity of the search is linear in $W_bD_b$.

%
%
\subsection{\MINTY{} in the Limits of Regularization $\gamma$}~\label{subsec:minty_limits}
In our proposed method, we penalize reliance on missingness to disjunctive linear rule models, controlling the emphasis on observed literals within rules, with the parameter $\gamma \geq 0$. In the low limit, $\gamma = 0$, MINTY is equivalent to a disjunctive linear rule model with zero-imputation. In stationary environments, where $p(X, M, Y)$ doesn't change between training and testing, for sufficiently large data sets, learning with $\gamma=0$ will result in the smallest error in general since this imposes the least constraints on the solution. This comes at the cost of reduced interpretability by relying on features with missing values in prediction. 

In the limit $\gamma \rightarrow \infty$, MINTY imposes a hard constraint that no rule should be included in the model unless it can be evaluated for every example in the training set without relying on imputed values, 
$\forall i,k : \rho_{ik} = 0$.
This could be appropriate in settings where there are \emph{some} features that are never missing and would be preferred over features that are predictive but rarely measured. However, if any configuration $m \in \{0,1\}^d$ of missing values is possible, MINTY will return an empty set of rules in the large-sample limit. 

\begin{thmobs}
If the all-missing configuration has positive marginal probability, $\exists \epsilon>0 : p(M=\boldsymbol{1}_d) > \epsilon$, the set of rules which have at least one literal measured for every example in the training set vanishes almost surely with growing number of samples $n$; there is no non-trivial GLRM $h$ with $\bar{\rho}(h)=0$ then. An important special case of this is the Missingness-Completely-At-Random (MCAR) mechanism~\citep{rubin1976inference} with missingness probability $q > \epsilon^{1/d}$.
\end{thmobs}
In other words, requiring \emph{perfect} variable redundancy through rules is too strict for many settings. Instead, we can aim to \emph{limit} or \emph{minimize} the reliance on missing values $\bar{\rho}$ by selecting a bounded $\gamma > 0$.

\subsection{Comparison With a Linear Model Trained on Complete Data}
In many applications, interpretable risk scores trained on complete cases are deployed in settings where features are occasionally missing, necessitating the imputation of missing values with a constant, often 0 for binary variables. One example is the APACHE family of clinical risk scores~\citep{afessa2005influence,haniffa2018performance}. It is natural to compare the bias of this approach to the bias of a model with inherently low reliance on missing values. Below, we do this for the case where the true outcome is a linear function and the variable set has a natural redundancy.

Assume that the outcome $Y$ is linear in $X \in \{0,1\}^d$ and has noise of bounded conditional variance,
$$
Y = \beta^\top X + \epsilon(X), \mbox{ where }\; \E[\epsilon \mid X] = 0,  \V[\epsilon \mid X] \leq \sigma^2~,
$$
with $\beta \in \mathbb{R}^d$. Next, assume that $X$ has the following structure. For each $X_i$ there is a paired ``replacement'' variable $X_{j(i)}$, with $j(j(i)) = i$, such that for $\delta \geq 0$, $p(X_i = X_{j(i)}) \geq 1-\delta$, and that whenever $X_i$ is missing,  $X_{j(i)}$ is observed, $M_i = 1 \Rightarrow M_{j(i)}=0$. Assume also that $\forall i, k \not\in \{i, j(i)\} : X_i \indep X_k$. 
    
\begin{thmprop}\label{prop:bias}
    Under the conditions above, there is a GLRM $h$ with $d$ two-variable rules $\{\zX_i \lor \zX_{j(i)} \}_{i=1}^d$, where $\zX_i = (1-M_i)X_i$, with expected the squared error 
    $$
    R(h) \leq \delta \|\beta\|_2^2 + \delta^2 \sum_{i,k\not\in\{i, j(i)\}}|\beta_i \beta_k| + \sigma^2~.
    $$
    Additionally, if $\beta_i \geq 0$ and $\E[X_iM_i] \geq \eta$ for all $i \in [d]$, using the ground truth $\beta$ (the ideal complete-case model) with zero-imputed features $\zX$ results in an expected sequared error bounded from below as 
    $$
    R(\beta) \geq \eta \|\beta\|_2^2 + \sigma^2~, 
    $$
    and a greater missingness reliance than the GLRM, $\bar{\rho}(\beta) \geq \bar{\rho}(h)$. Thus, with $a = \|\beta\|_2^2/\sum_{i,k\not\in\{i, j(i)\}}|\beta_i \beta_k|$, the GLRM is preferred when $\delta < (\sqrt{a^2 + 4\eta} - a)/2$.
\end{thmprop}
A proof is given in the supplement. 

By Proposition~\ref{prop:bias}, there are data-generating processes for which a disjunctive GLRM has a strictly smaller risk and smaller reliance on features with missing values than the ground-truth linear rule model used with zero imputation. For simplicity, the result is written for rules involving pairs of variables that are internally strongly correlated and independent of other pairs but can be generalized to disjunctions of variables in cliques of any size with the same property.  

%
%
\section{Empirical Study}%
\label{sec:experiments}%
We evaluate the proposed \MINTY{} algorithm\footnote{The code to reproduce the experiments at \url{https://github.com/Healthy-AI/_minty}} 
on synthetic and real-world data, aiming to answer three main questions: i) How well can we learn rules when covariates are missing at training and test time? ii) How does the accuracy of \MINTY{} compare to baseline models; iii) How does regularizing reliance on missing values affect performance and interpretability?

%
%
\subsection{Experimental Setup}
In our experiments, we solve the rule-generation subproblem of \MINTY{} using beam search, as described in Section~\ref{sec:minty_opt}. In the supplement,  
we use a small synthetic data set to show that the predictive performance differs only minimally compared to solving the ILP in \eqref{eq:relaxed_minty} exactly using Gurobi~\citep{gurobi}. To find optimal coefficients $\beta$, given rule definitions $S$, we use the \LASSO{} implementation in scikit-learn~\citep{sklearn_api}, re-weighting covariates to achieve variable-specific regularization. 

The objective function regularizes each rule $z_{\cdot k}$ with strength $\lambda_k = \lambda_0 + \lambda_1\|z_{\cdot k}\|_0$, limiting the reliance on missingness by minimizing the number of rules using zero-imputed features. The values of $\lambda_0$ and $\lambda_1$ range within $[10^{-3}, 0.1]$. We choose their best values through a grid search based on their validation set performance.
The values for $\gamma$ were chosen from $[0, 10^{-7}, 10^{-3}, 0.01, 0.1, 10000]$. We present the result for several values of $\gamma$, to illustrate the tradeoff between performance and reliance on missing values. 

We compare \MINTY{} to the following baseline methods: Imputation + logistic regression \LASSO{}, Imputation + Decision Tree \DT{}, XGBoost (\XGB{}), where missing values are supported by default~\citep{chen2019package}. Last, we compare the NeuMiss network, \NEUMISS{} that proposes a new type of non-linearity: the multiplication by the missingness indicator~\citep{morvan2020neumiss}.
Hyperparameters are chosen on the validation set performance. For imputation, we use zero ($I_0$), or the iterative imputation method, called MICE ($I_{mice}$) from SciKit-Learn~\citep{scikit-learn_imp, buren}, that replaces missing values with multiple imputations using a regression model. MICE was performed over 5 iterations. Details about method implementations, hyperparameters, and evaluation metrics are given in Supplement~\ref{app:baselines}. 

We report the mean square error (MSE) and the $R^2$ score (coefficient of determination) for all methods. Statistical uncertainty of the metrics is estimated using standard $95\%$-confidence intervals over the test set. Additionally, we estimate the reliance on features with missing values, $\bar{\rho}$ of all methods on the test sets. For \LASSO{}, this counts the fraction of observations with missing values among the features with non-zero coefficients. For \DT{}, we report the fraction of inputs with a feature that is both missing (and thus imputed), and used in a split to decide that inputs prediction. For \XGB{}, we do the same, but count observations for which \emph{any} of the trees rely on a missing value. \NEUMISS{} uses all variables for prediction, and so $\bar{\rho}$ measures the fraction of observation with at least one missing value. For \MINTY{}, we define $\bar{\rho}$ as explained in Section~\ref{sec:methodoloy}.

\paragraph{Real-world Data Sets}
We used three different data sets for regression tasks. The first data set, \textit{ADNI}, is sourced from the Alzheimer's Disease Neuroimaging Initiative (ADNI) database and involves predicting the outcome of the ADAS13 cognitive test at a 2-year follow-up based on baseline data. The data set, \textit{Life}, aims to predict life expectancy from various factors, including immunization, mortality, economic, and social factors~\citep{owidlifeexpectancy}. Last, the data set \textit{Housing} involves predicting property prices in Ames, Iowa, using physical attributes and geographical features~\citep{de2011ames}. The data sets are discretized and used with binary data in the baselines and for \MINTY{}. More details can be found in the Appendix. We split the data randomly into a test set (20\%) and a training set (80\%), and then withhold a validation portion (20\%) of the training set for selecting hyperparameters. We average results over 10 seeds. 

\paragraph{Missing Values}\ADNI{} has incomplete entries natively, indicated in results as ``(Natural)'' missingness. We added missing values to \Life{} and \Housing{} according to the Missing Completely at Random (MCAR) mechanism, where the probability that a feature $X_j$ has a missing value is $q$, independent of other variables. In our experiments, we set $q$ to 0.1. The same mechanism is used both during training and testing. 

\paragraph{Synthetic Data}
In the Supplement, we also apply our algorithms for synthetic data where $n=5000$ samples of $c=30$ features are drawn from independent Bernoulli variables. Then, for each variable $X_i$, $i\in [c]$ a ``replacement variable'' $X_{c+i}$, is added which has the same value as $X_i$ with probability $0.9$. The outcome $Y$ is a linear combination of all features with added noise. Missingness can be added either with either MCAR, Missing-Not-At-Random (MNAR), or Missing-At-Random (MAR) mechanisms using the implementation by~\citet{Mayer2019}. 

\subsection{Results}\label{sec:results}
We report the predictive performance of all models and their reliance on features with missing values in Tables~\ref{tab:results_real-world_adni_housing}--\ref{tab:results_real-world_life}, and comment on their interpretability. 

\begin{table*}[t]
\centering
\caption{Performance results for the real-world data sets \ADNI{} and \Housing{}. For \MINTY{} using \ADNI{} we use $\lambda_0 = 0.001, \lambda=_1 = 0.01$, and for \Housing{} we choose $\lambda_{0}=0.01$, $\lambda_{0}=0.001$ based on a 0.1 missingness proportion in the data. $^*$Training of \NEUMISS{} for \Housing{} only finished 3 out of 10 seeds.} \label{tab:results_real-world_adni_housing}%
\begin{small}%
\begin{tabular}{l|lll|lll}
\toprule
& \multicolumn{3}{c|}{ADNI (Natural)} & \multicolumn{3}{c}{HOUSING (MCAR)} \\
\textbf{Model} &\textbf{$R^2$} & \textbf{MSE} & $\bar{\rho}$ &\textbf{$R^2$} & \textbf{MSE} & $\bar{\rho}$  \\
\midrule 
\LASSO{}$_{I_{mice}}$   &  0.38 (0.29, 0.47) & 0.63 (0.51, 0.78) & 0.55  &  0.58 (0.50, 0.65)&  0.44 (0.33, 0.54)  & 0.99   \\
\DT{}$_{I_0}$ &  0.57 (0.49, 0.65) & 0.44 (0.33, 0.55)  &  0.08  &  0.61 (0.54,  0.68) & 0.40 (0.30, 0.50) &  0.22 \\
\XGB{} &  0.58 (0.48, 0.64) & 0.45 (0.33,  0.56)  & 0.55  & 0.69  (0.63, 0.76) &  0.31 (0.22, 0.40) & 0.84 \\ 
\NEUMISS{} & 0.52 (0.44, 0.60) & 0.49 (0.37, 0.61) & 0.55 & 0.69 (0.62,   0.75)$^*$&  0.34 (0.24, 0.43)$^*$ &  0.99$^*$\\
\MINTY{}$_{\gamma = 0}$  & 0.64 (0.56, 0.70) &  0.37 (0.27, 0.47) &  0.40 & 0.72 (0.66, 0.78) & 0.29 (0.20, 0.37)  & 0.61  \\
\MINTY{}$_{\gamma = 0.01 \text{(A)}, \gamma = 0.1 \text{(H)}}$  & 0.63 (0.56, 0.70) & 0.38  (0.27, 0.48) & 0.28 & 0.72 (0.66, 0.78) &  0.29 (0.20, 0.37) & 0.48 \\
\MINTY{}$_{\gamma =1e4}$  & 0.62 (0.55, 0.70) &  0.38 (0.27,   0.48) & 0.0 & 0.46 (0.38, 0.55) &  0.55 (0.43, 0.67) & 0.0  \\
\bottomrule
\end{tabular}
\end{small}
\end{table*}


\begin{table}[t]
\centering
\begin{small}%
\caption{Performance results for real-world data set \Life{} with $\lambda_0 = 0.01$, $\lambda_1= 0.01$ for \MINTY{}. The missingness proportion is 0.1. For the synthetic data, we used $\lambda_0 = 0.01$, $\lambda_1= 0.01$ for the extreme cases of \MINTY{} and $\lambda_0 = 0.001$ for $\gamma=0.1$. The missingness proposition of 0.1 together with 0.1 for replacement disagreement probability}%
\label{tab:results_real-world_life}%
\tabcolsep=0.11cm
\begin{tabular}{llll}
\toprule
\multicolumn{4}{c}{LIFE (MCAR)} \\
\textbf{Model} &\textbf{$R^2$} & \textbf{MSE} & $\bar{\rho}$ \\
\midrule 
\LASSO{}$_{I_0}$   & 0.87 (0.84, 0.90) &  11.3 (10.9, 11.7)  &  0.71  \\
\DT{}$_{I_0}$ & 0.88 (0.85, 0.91) & 10.4 (10.0, 10.8) &  0.32  \\
\XGB{} &  0.94 (0.92, 0.96) &  5.14 (4.87, 5.40)&  0.82   \\ 
\NEUMISS{} & 0.82  (0.79, 0.85) & 15.9 (15.4, 16.3) &  0.82 \\
\MINTY{}$_{\gamma=0}$ &  0.91  (0.88, 0.93) & 8.22 (7.89,   8.55) & 0.73    \\
\MINTY{}$_{\gamma=0.5}$   &  0.90 (0.88, 0.93)& 8.37 (8.03, 8.70) & 0.61  \\
\MINTY{}$_{\gamma=1e4}$ & 0.00  (-0.08,  0.08) & 88.0 (86.9, 89.1) & 0.00   \\
\bottomrule
\end{tabular}%
\end{small}%
\end{table}

\MINTY{} achieves the best (\ADNI{}, \Housing{}) or second-best (\Life{}) held-out predictive performance (high $R^2$, low MSE), while relying substantially less on features with missing values in the test set (smaller $\bar{\rho}$) than models with comparable predictive accuracy. On \ADNI{}, a \MINTY{} model with $\bar{\rho}=0$ achieves better $R^2$ than \XGB{} and \NEUMISS{} models for which more than 50\% of the test samples must use default rules or be imputed, respectively. This confirms that it is possible to learn to avoid imputation to a large degree while maintaining a competitive model. We see similar results on \Housing{} and \Life{}, despite the missingness being unstructured in these examples (MCAR). In Supplement Tables~\ref{app_tab:results_synth_MCAR}--\ref{app_tab:results_synth_MAR_MNAR}, we compare all models on synthetic data in MCAR, MAR, and MNAR settings and see that $\MINTY{\gamma=0.01}$ is among the best-performing models regardless of the missingness mechanism. 

We note that \XGB{} (tree ensemble) and \NEUMISS{} (multi-layer neural network) support prediction with missing values natively and perform well in all tasks, but can be difficult to interpret due to their large size and/or black-box nature. In Supplement Figure~\ref{app:fig_complexity}, we report the $R^2$ values on \ADNI{}, together with estimator-specific measures of complexity. These results, the results in Tables~\ref{tab:results_real-world_adni_housing}--\ref{tab:results_real-world_life}, and the model description in Table~\ref{tab:customized_descrptions_ADNI} confirm that \MINTY{} can be used to learn (more) interpretable models while handling missing values at test time. Notably, \DT{} also relies less on missing values than other baselines, simply because not every variable will be used to compute the prediction for every test instance. This suggests that building trees with regularization for $\rho$ could be a useful future investigation.

\paragraph{The Impact of Regularizing $\bar{\rho}$ with $\gamma>0$} 

For all data sets, we show that there are values of $\gamma>0$ such that $\MINTY{}_{\gamma>0}$ and $\MINTY{}_{\gamma=0}$ differ only slightly in their $R^2$ and MSE values but where $\MINTY{}_{\gamma>0}$ shows substantially lower reliance on imputation. For example, on \Housing{}, $\MINTY{}_{\gamma=0.1}$ achieves the same $R^2$ as $\MINTY{}_{\gamma=0}$ but with reliance $\bar{\rho}=0.48$ compared to $\bar{\rho}=0.61$ for the unregularized model. As remarked previously, achieving $\bar{\rho}=0$ with non-trivial predictive performance is not always possible: on \Life{}, the upper extreme of $\gamma = 10000$ results in an uninformative model, since there were no rules which were always determined by observed values other than the intercept.

In Figure~\ref{fig:reliance}, we show the results of \MINTY{} with $\gamma \in [10^{-6}, 1000]$ swept over a log-scale of 20 values from this set. For $\gamma = 1000$, the model is disallowed any use of missing values in the rules ($\bar{\rho} = 0$), which leads to worse predictive performance (bottom left in Figure). In the top right corner, we set results for $\gamma =0$ which results in the best predictive performance, but the highest reliance on missing values. Regularizing reliance on missing values moderately ($\gamma=0.01$) leads to a good balance of predictive accuracy $(R^2=0.63)$ and reliance on imputation $(\bar{\rho} = 0.28)$.

\begin{figure}[t]
\includegraphics[width=8cm]{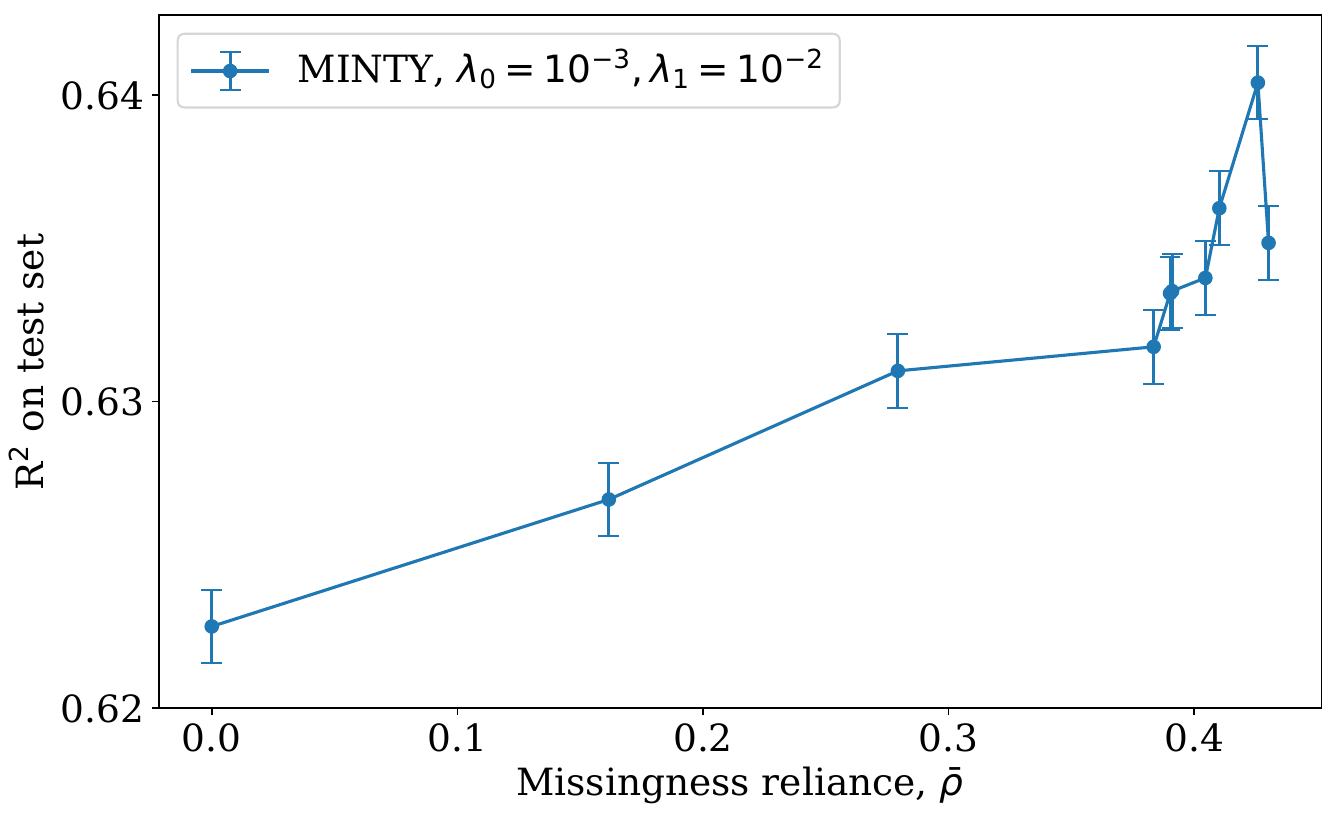}
\caption{Predictive performance ($R^2$) and reliance on features with missing values $\bar{\rho}$ on \ADNI{} for \MINTY{} with $\gamma$ chosen from a log scale over $[10^{-6}, 10^3]$.}
\label{fig:reliance}
\end{figure}

\subsection{Interpreting Learned Rules on \ADNI{}}
In Table~\ref{tab:customized_descrptions_ADNI}, we visualize the models learned by \MINTY{} on \ADNI{}, in the style of risk scores used in medicine or criminal justice, see, e.g., \citet{ustun2019learning}. On the left are rule definitions and on the right, their coefficients---the score added if the rule is true. The scores for each active rule are summed together with the intercept to form a prediction. The top table represents the learned set of rules using \MINTY{}$_{\gamma=0}$ and the bottom one for \MINTY{}$_{\gamma=0.01}$. 

In the \ADNI{} task, the goal is to predict the cognitive decline measured by a change in the cognitive test score ADAS13 (high score means low cognitive ability, a positive change means deteriorating ability) from baseline to a 2-year follow-up. The learned coefficients match expectations as, for example, diagnoses for Alzheimer's disease (AD) or mild cognitive impairment (LMCI) are associated with higher cognitive decline (positive coefficients). Similarly, \texttt{MMSE $\geq$ 29} (normal cognitive ability) is associated with smaller decline in ADAS13 (negative coefficient).

The two models with $\gamma=0$ and $\gamma=0.01$ learn similar rules with similar coefficients but with different reliance on features with missing values ($\bar{\rho}=0.40$ vs $\bar{\rho}=0.28$). The rules, \texttt{TAU $\leq$ 191.1 OR Hippocampus $\geq$ 7721.0}  and \texttt{FDG $\leq$ 1.163} are not included in the second model ($\gamma=0.01)$, since they are missing for $0.33\%$ and $0.27\%$ of all individuals in the data set. By using a higher $\gamma$ we achieve a more robust solution with less dependence on imputed values. 

For \MINTY{}$_{\gamma=0.1}$, which achieves $\bar{\rho} = 0$, shared in Table~\ref{tab:customized_descrptions_ADNI_minty01} in the Supplement, we see that the learned rules contain mostly features that are \emph{always} measured such as demographics and cognitive test scores, following the constraint that rules should not be included unless it can be evaluated for every example. We also show an example in Table~\ref{app_tab:customized_descrptions_synth} in the Supplement, where the true rules produced by synthetic data are recovered. 

\begin{table}[t]
\centering
\caption{\MINTY{} models learned on \ADNI{} using $\gamma=0$ (top) and $\gamma=0.01$ (bottom). The $R^2$ for the two models were $0.64$ and $0.63$ respectively, the latter with smaller reliance on features with missing values ($\bar{\rho}=0.28$ vs $\bar{\rho}=0.40$). Two rules in the top model are not in the bottom model due to more frequent missingness; the bottom model adds two rules with less missingness.}%
\label{tab:customized_descrptions_ADNI}%
\begin{small}
\begin{tabular}{l| c}
\toprule
Rules by \MINTY{} with $\gamma = 0$ & Coeff.  \\
\midrule
AD diagnosis OR  LMCI diagnosis & +0.35\\
MMSE $\leq$ 26.0 OR  LMCI diagnosis   &  +0.23 \\ 
LDELTOTAL $\leq$ 3.0     &  +0.63 \\
AD diagnosis      &   +0.65 \\
Hippocampus $\leq$ 6071.0 OR Sex = Male  & +0.18 \\
MMSE $\geq$ 29.0 &   $-$0.16 \\
Entorhinal $\leq$ 3022.0  &     +0.18 \\
LDELTOTAL score $3 - 8$   & +0.27 \\
{\color{red} TAU $\leq$ 191.1 OR Hippocampus $\geq$ 7721.0}  & {\color{red} $-0.19$ } \\
{\color{red} FDG $\leq$ 1.163}    & {\color{red}+0.17} \\
\midrule
Intercept &  -0.57\\ 
\bottomrule
\multicolumn{2}{c}{} \\
\toprule
Rules by \MINTY{} with $\gamma = 0.01$ & Coeff.  \\
\midrule
AD diagnosis OR LMCI diagnosis & +0.36\\
MMSE $\leq$ 26.0 OR  LMCI diagnosis   &  +0.22 \\ 
LDELTOTAL $\leq$ 3.0     &  +0.67\\
AD diagnosis      &   +0.68 \\
Hippocampus $\leq$ 6071.0 OR Sex = Male  & +0.19 \\
MMSE $\geq 29$ & $-$0.17 \\
Entorhinal $\leq$ 3022.0  &     +0.17 \\
LDELTOTAL score $\in [3,8]$   & +0.28 \\
{\color{blue}Hippocampus $\geq$ 7721.0}  & {\color{blue}-0.16} \\
{\color{blue} APOE4 $ = 1$} & {\color{blue}+0.08} \\
\midrule
Intercept &  -0.61\\ 
\bottomrule
\end{tabular}%
\end{small}%
\end{table}

%
%
\section{Related Work}
The rich literature on learning from data with missing values, see e.g.,~\citet{little2019statistical,Mayer2019}, studies both a) settings in which complete inputs are expected at test time but have missing values during training, and b) predictive settings where missing values are expected also during testing~\citep{josse2019consistency}. Studies of the first category have produced impressive results that give inference guarantees under different missingness mechanisms, such as MCAR, MAR, MNAR~\citep{rubin1976inference} and have often focused on imputing missing values with model-based techniques~\citep{buren}. Our work falls firmly in the second category, born out of supervised learning: rather than assuming that a particular mechanism generated missingness, we assume that the mechanism is preserved at test time~\citep{josse2019consistency}. 

Two common strategies in our setting are to i) impute-then-regress---to impute missing values and proceed as if they were observed, or ii) build models that explicitly depend on the missingness mask $M$, indicators for missing values~\citep{little2019statistical}. The former approach can introduce avoidable bias even with powerful imputation methods in the setting where values are missing in the same distribution during testing as during training ~\citep{le2021sa}. \citet{josse2019consistency} showed that pairing constant imputations, for example with 0, with a sufficiently expressive model leads to consistent learning. A drawback of this is that the optimal imputation or regression models are often complex and challenging to interpret.

The second strategy resulted in diverse methods, many of which incorporate the missingness mask in deep learning~\citep{bengio1995recurrent,che2018recurrent, lemorvan20a_linear, nazabal2020handling}. More recently, NeuMiss networks~\citep{morvan2020neumiss} introduced a deep neural network architecture that applies a nonlinearity based on the missingness mask to learn Bayes optimal linear predictive models. Another approach is the so-called Missing Incorporated in Attribute (MIA)~\citep{twala2008good} which uses missingness itself as a splitting criterion in tree learning, as used by e.g., XGBoost~\citep{chen2016xgboost}. A drawback of these methods is that they are difficult to interpret due to their complexity. In concurrent work, \citet{chen2023missing} addressed missing values with explainable machine learning but focused on a different model class from ours, using explainable boosting machines (EBMs) to gain insights without relying on imputation or specific missingness mechanisms. 
%
%
\section{Conclusion}
We have proposed \MINTY{}, a generalized linear rule model that mitigates reliance on missing values by a) using disjunctive rules whose values can be computed as long as one of its literals is observed and true, and b) regularizing the inclusion of rules whose values can frequently not be determined. We demonstrated in experiments on real-world data that \MINTY{} often has similar accuracy to black-box estimators and outperforms competitive baselines while maintaining interpretability and minimizing the reliance on missing values. 
%
\section*{Acknowledgements}
This work was partly supported by WASP (Wallenberg AI, Autonomous Systems and Software Program) funded by the Knut and Alice Wallenberg foundation. 

The computations were enabled by resources provided by the Swedish National Infrastructure for Computing (SNIC) at Chalmers Centre for Computational Science and Engineering (C3SE) partially funded by the Swedish Research Council through grant agreement no. 2018-05973.

%
%
%

\clearpage
\bibliography{references} 

\clearpage
%
%
\appendix

\section{Supplementary Material}

\subsection{Computing Infrastructure}
The computations required resources of 4 compute nodes using two Intel Xeon Gold 6130 CPUS with 32 CPU cores and 384 GiB memory (RAM). Moreover, a local disk with the type and size of SSSD 240GB with a local disk, usable area for jobs including 210 GiB was used. 
Inital experiments are run on a Macbook using macOS Monterey with a  2,6 GHz 6-Core Intel Core i7 processor. 

\subsection{Baseline models}~\label{app:baselines}
The baselines are trained by the following parameters. The best values for these hyperparameters are chosen based on the validation test set. 

\paragraph{\LASSO{}:} The values of alpha indicating a $\ell_1$ regularization term on weights range within $[0.1, 0.6]$, where increasing this value will make model more conservative. We allow to fit an intercept and set the precompute parameter to \textit{TRUE} to get the precomputed Gram matrix to speed up calculations~\citep{sklearn_api}. \LASSO{} is trained with zero and MICE imputation and chosen based on the validation performance. 
\paragraph{\XGB{}{}:} 
In \XGB{} we range the learning rate ($\eta$) between $[0.2, 0.3]$ where the shrinking step size is used in the update to prevent overfitting. After each boosting step, we can directly get the weights of new features, and $\eta$ shrinks the feature weights to make the boosting process more conservative.
The maximum depth of the trees is set to 4 since increasing this value will make the model more complex and more likely to overfit~\citep{chen2016xgboost}. 
The hyperparameters $\lambda$ represent the $\ell 2$ regularization term on weights and $\alpha$ indicates the $\ell 1$ regularization term. We set $\lambda$ to 0.5 and $\alpha$ to 0.2. Increasing this value will make a model more conservative. \XGB{} does not rely on imputation and chooses a default direction for missing values learned during training. 
\paragraph{\DT{}{}:}
For \DT{} we set the criterion to measure the quality of a split using the 'squared error' and used 'best' as the strategy to choose the split at each node. The minimum number of samples per leaf can range between [10, 20, 50].  A node will be split if this split induces a decrease of the impurity greater than or equal to 0.1. Complexity parameter 'ccp alpha' is used for Minimal Cost-Complexity Pruning where the subtree with the largest cost complexity that is smaller than 0.005 will be chosen~\citep{sklearn_api}. We use zero imputation for all  \DT{}s. 
\paragraph{\NEUMISS{}{}:} For \NEUMISS{} models we define the dimension of inputs and outputs of the NeuMiss block (n-features), set the number of layers (Neumann iterations) in the NeuMiss block (depth) and range the number of hidden layers in the MLP (mlp depth) between [3,5] and set the width of the MLP (mlp width) to 30. If 'None' takes the width of the MLP will be the same as the number of covariates of a data set~\citep{le2020neumiss}. 

\subsection{Real world data sets}~\label{appendix:real_world_data}
\paragraph{ADNI}
The data is obtained from the publicly available Alzheimer's Disease Neuroimaging Initiative (ADNI) database. ADNI collects clinical data, neuroimaging and genetic data~\citep{ADNI}. The regression task aims to predict the outcome of the ADAS13 (Alzheimer's Disease Assessment Scale)~\citep{mofrad2021cognitive} cognitive test at a 2-year follow-up based on available data at baseline.

\paragraph{Life}
The data set related to \textit{life} expectancy, has been collected from the WHO data repository website\citep{world2021ghe}, and its corresponding economic data was collected from the United Nations website. The data can be publicly accessed trough~\citep{owidlifeexpectancy}. In a regression task, we aim to predict the life expectancy in years from 193 countries considering data from the years 2000-2025. The final dataset consists of 20 columns and 2864 samples where all predicting variables were then divided into several broad categories: immunization factors, mortality factors, economic factors, and social factors.

\paragraph{Housing}
The Ames \textit{housing} data set was obtained from 
(\href{http://www.kaggle.com}{http://www.kaggle.com}) and describes the selling price of individual
properties, various features, and details of each home in Ames, Iowa, USA from 2006 to 2010~\citep{de2011ames}. We selected 51 variables on the quality and quantity of physical attributes of a property such as measurements of area dimensions for each observation, including the sizes of lots, rooms, porches, and garages or some geographical categorical features related to profiling properties and the neighborhood. In a regression task, we used 1460 observations. 

\subsection{Additional results}\label{appendix}
\begin{table}[ht]
\caption{Performance results for Synthetic data of 500 samples and 15 covariates over 10 seeds using Gurobi or beam-search as a solver for the optimization. $\lambda_0 = 0.01$  and $\lambda_1 = 0.01$ were chosen. } \label{app_tab:supp_ilp_beam}
\begin{center}
\begin{tabular}{l|lll|lll}
& \multicolumn{3}{c}{Synthetic (MNAR),  \textit{ILP}} & \multicolumn{3}{c}{Synthetic (MNAR), \textit{beam search}}  \\
\textbf{Model} &\textbf{$R^2$} & \textbf{MSE} & $\bar{\rho}$ &\textbf{$R^2$} & \textbf{MSE} & $\bar{\rho}$ \\
\hline \\
\MINTY{}$_{\gamma = 0}$  & 0.72 (0.61, 0.82) &  1.32 (1.00, 1.64)  &   0.36  & 0.73  (0.62,  0.83) &1.30  (0.98,  1.61)  &   0.36\\
\MINTY{}$_{\gamma=0.01}$ &  0.72  (0.61,  0.82) &  1.32 (1.00, 1.64) & 0.34 & 0.73 (0.63, 0.83) & 1.29 (0.98, 1.60) & 0.26\\
\MINTY{}$_{\gamma = 10000}$ & -0.00 (-0.21, 0.19) & 4.71 (4.11, 5.32)  &  0.03 & -0.01 (-0.21,  0.19) &4.74 (4.14, 5.35)   &  0.00\\
\end{tabular}
\end{center}
\end{table}

We show in Table~\ref{app_tab:supp_ilp_beam} the comparison between the optimal solution found by the Gurobi~\citep{gurobi} solver (left in table), and the approximate solutions using a heuristic beam search algorithm. We see that when using beam-search, we achieve almost the same results as with Gurobi.

\begin{table*}[t]
\begin{small}
\caption{Performance results for Synthetic data as described in, with 10 iterations and 7000 samples and 15 columns.} \label{app_tab:results_synth_MCAR}
\begin{center}
\begin{tabular}{l|lll}
&  \multicolumn{3}{c}{SYNTH (MCAR)} \\
\textbf{Model} &\textbf{$R^2$} & \textbf{MSE} & $\bar{\rho}$  \\
\hline \\ 
\LASSO{}$_{I_0}$  & 0.47 (0.43, 0.52)  &     4.64 (4.46,  4.82)  &0.82 \\
\DT{}$_{I_0}$  & 0.36 (0.31, 0.41) & 5.60 (5.39, 5.80) & 0.44 \\
\XGB{} & 0.76 (0.73, 0.79) & 2.04 (1.91, 2.16) & 0.99   \\ 
\NEUMISS{} &  0.75 (0.72, 0.78) & 2.07 (1.95, 2.20) & 0.99\\
\MINTY{}$_{\gamma=0}$ &  0.66 (0.63, 0.70) & 2.94  (2.79,  3.09) &  0.81  \\
\MINTY{}$_{\gamma=0.1}$ & 0.66 (0.62, 0.69) &  2.99 (2.84, 3.14)  &  0.80 \\
\MINTY{}$_{\gamma=10000}$  &  -0.00  (-0.06,  0.06) &  8.98 (8.72, 9.24)& 0.00 \\
\end{tabular}
\end{center}
\end{small}
\end{table*}

\begin{table*}[t]
\caption{Performance results for synthetic data as described in ~\ref{app_tab:customized_descrptions_synth}, with 10 iterations and 7000 samples and 15 columns.} \label{app_tab:results_synth_MAR_MNAR}
\begin{center}
\begin{tabular}{l|lll|lll}
& \multicolumn{3}{c}{Synth (MAR), $\lambda_0=0.001$, $\lambda_1=0.01$} & \multicolumn{3}{c}{Synth (MNAR), $\lambda_0=0.01$, $\lambda_1=0.01$}\\
\textbf{Model} &\textbf{$R^2$} & \textbf{MSE} & $\bar{\rho}$ & \textbf{$R^2$} & \textbf{MSE} & $\bar{\rho}$\\
\hline \\ 
\LASSO{}$_{I_0}$& 0.55 (0.51, 0.59) &  3.81 (3.65, 3.98) &  0.65 & 0.51 (0.47,  0.55) &  4.13 (3.96,  4.31) &   0.70  \\
\DT{}$_{I_0}$ & 0.43 (0.38, 0.48) &  4.94 (4.75, 5.14)   &   0.23 & 0.41   (0.36, 0.45) &  5.18  (4.98, 5.37)  &  0.29  \\ 
\XGB{} & 0.80  (0.77, 0.83)  &  1.64 (1.52, 1.75)  &    0.92 & 0.79  (0.76, 0.82) & 1.76 (1.64, 1.87)&  0.96  \\ 
\NEUMISS{} & 0.81 (0.79,  0.84)& 1.52 (1.41,  1.63) &  0.92 & 0.75 (0.71,   0.78)&  1.97 (1.85, 2.09) &   0.96 \\
\MINTY{}$_{\gamma = 0}$ & 0.69  (0.65,  0.72) & 2.69  (2.54,  2.83)  &  0.51 & 0.67  (0.63,  0.70)& 2.69  (2.55,  2.83) &   0.64\\
\MINTY{}$_{\gamma=0.01}$ &  0.69 (0.65,  0.72) & 2.69 (2.55,   2.83)   &  0.49 & 0.66 (0.63,  0.70) & 2.89  (2.74, 3.04)  & 0.64 \\
\MINTY{}$_{\gamma = 10000 }$  & 0.25  (0.20, 0.31) & 6.60 (6.38,   6.83)   &  0.00 & -0.00   (-0.06, 0.06)  & 8.81 (8.55, 9.07) &   0.00\\
\end{tabular}
\end{center}
\end{table*}

\paragraph{Complexity vs. predictiveness}
Results are shown in Figure~\ref{app:fig_complexity}, comparing the $R^2$s with estimator-specific complexity measurements. We observe that \MINTY{}$_{\gamma=0.1}$ balances the trade-off between good predictive performance with a small number of non-zero coefficients which in turn ensures lower model complexity (15 coefficients). One reason why \MINTY{}$_{\gamma=0.10}$  performs better than \MINTY{}$_{\gamma=0}$ (essentially zero-imputation) is that it can choose from a bigger set of rules. However, this also increases the reliance on imputed values and some level of bias in the model. \NEUMISS{} which shows the lowest complexity, however, depends on imputation, and cannot be interpreted due to its black-box nature. Similary for \DT{}, which performs the best on the ADNI data but perhaps lacks some interpretability with almost 40 numbers of leaves. In a \DT{}, neighboring leaves are similar to each other as they share the path in the tree. As the number of leaves increases, variance in the performance increases and perhaps compromises interpretability. \XGB{} achieves consistent performance across estimators, but could be difficult to interpret with a larger number of estimators (and an even larger number of parameters).  While \LASSO{} is the simplest model, its performance is the lowest.  

\begin{figure}[ht]
\centering
\vspace{.1in}
\includegraphics[width=10cm]{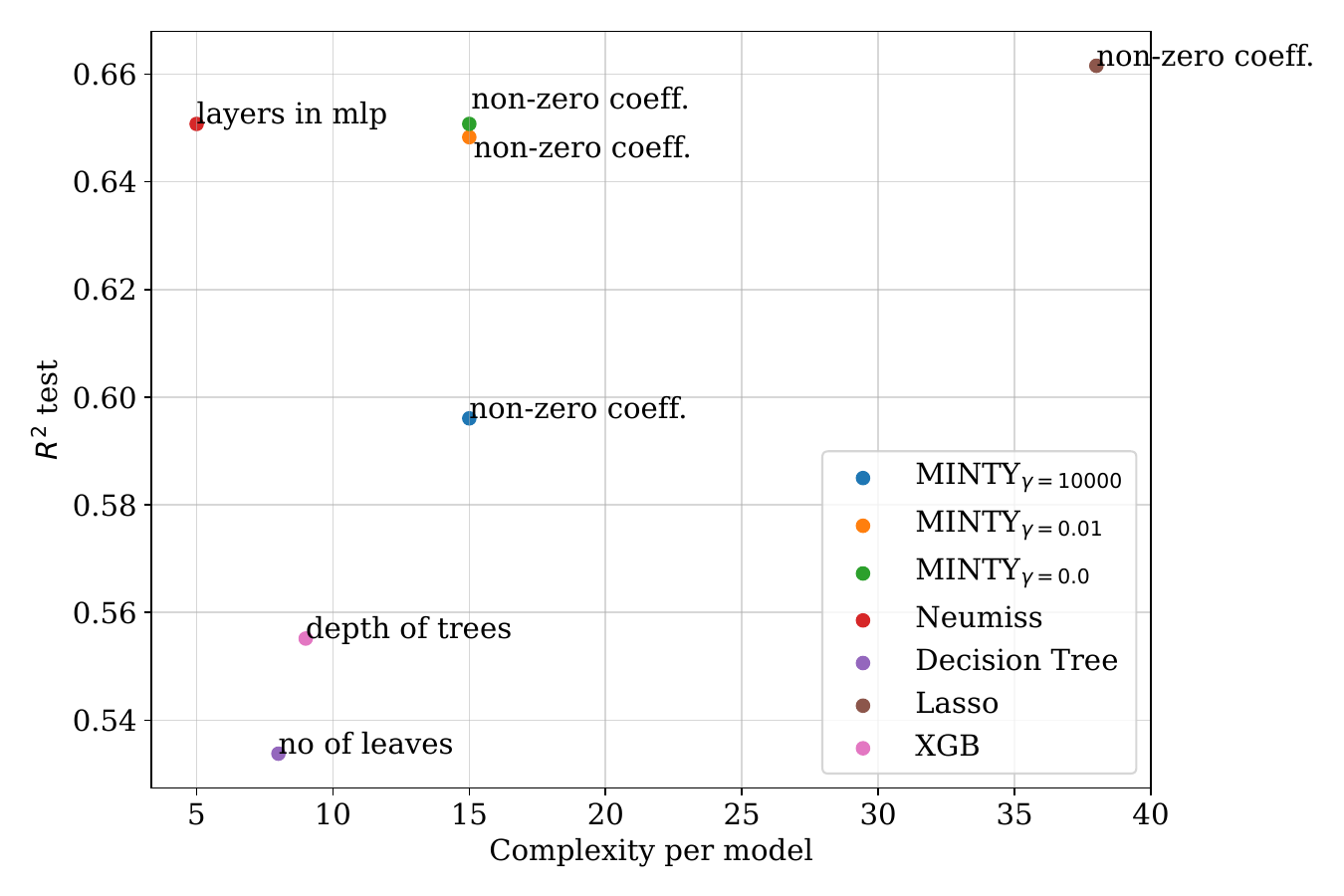}
\vspace{.2in}
\caption{Performance against complexity measurement on \ADNI{} data. As a criterion for complexity, we use for \MINTY{} models and \LASSO{} the number of non-zero coefficients achieved by regularisation. \NEUMISS{} does not aim at a sparse solution and therefore we give the complexity by the number of layers in the MLP network. Note, that there might be more parameters to optimize for. The complexity for \XGB{} is defined by the depth of the trees, and for \DT{} we describe the number of leaves.}\label{app:fig_complexity}
\end{figure}

\paragraph{Customized Rules}
\begin{table*}[t]
\centering
\caption{Customized rule sets for predictions using ADNI data using $\gamma=0$ (top) and $\gamma=0.01$ (bottom). The $R^2$ for the two models were $.64$ and $.63$ respectively, but the latter had significantly smaller reliance on features with missing values ($\bar{\rho}=0.28$ vs $\bar{\rho}=0.40$). The red rules in the top model are not present in the bottom and have larger missingness in the data. The blue rules in the bottom model are not present in the top and have less missingness.}  
\label{tab:customized_descrptions_ADNI_minty01}
\vskip 0.08in
\begin{center}
\begin{small}
\tabcolsep=0.04cm
\begin{tabular}{l|c}
\toprule
Learned Rules by \MINTY{}$_{\gamma =0.1}$ & Coeff.  \\
\midrule
 LDELTOTAL $\in [8-12]$ OR LDELTOTAL $\geq12$ OR Cognitive normal diagnosis &     -0.81 \\
 LDELTOTAL $\leq$ 3.0 OR LDELTOTAL $\in [8-12]$ OR Alzheimer's diagnosis  &   +0.42 \\
LDELTOTAL $\geq12$ OR MMSE $\leq 26$ &  +0.13 \\
AGE $\geq 78.5$ OR  MMSE $\in [26-28]$ OR Alzheimer's diagnosis   &  +0.35 \\
MMSE $\leq 26$ OR SEX = Male & +0.17 \\
AGE $\in [73.5-78.5]$ OR APOE4$=2.0$ OR Alzheimer's diagnosis &  +0.25 \\
68.9  AGE $\in [68.9-73.5]$ OR MMSE $\in [26-28]$ OR Alzheimer's diagnosis      &   +0.22 \\
MMSE $\in [26-28]$ OR MMSE $\geq 29.0$ OR Race=Black  & -0.22 \\
LDELTOTAL $\leq$ 3.0 OR APOE4 = 1.0  &   +0.14 \\
MMSE $\in [26-28]$ OR MMSE $\in [28-29]$ OR  Cognitive normal diagnosis OR EMCI diagnosis &  -0.15 \\
\midrule
Intercept &  -0.09\\ 
\bottomrule
\end{tabular}
\end{small}
\end{center}
\vskip -0.1in
\end{table*}
We use simulated data $X_{sim}$ by sampling $n\times d$ independent binary input features. However, we add some conditional dependence between columns 0 and 4 to illustrate the process of generating replacement variables focusing on predictive performance and interpretability. 
Each element of $X_{sim}$ is randomly set to 0 or 1 based on whether a random value drawn from a standard normal distribution is greater than 0.
The outcome $Y$ is based on the values in columns 0 and 4 of $X_{sim}$, adding a constant term of 1 and some random noise drawn from a standard normal distribution. 

\begin{table}[t]
\caption{Customized rule sets for predictions based on the ground true rule set (Top table). Learned rules set with corresponding coefficients in the bottom table are based on \MINTY{}. The results are based on a generated data set with $n = 7000$ samples and a $p_{miss}=0.1$}
\vskip 0.15in
\begin{center}
\begin{small}
\begin{sc}
\begin{tabular}{l| c}
\toprule
True Rules & Coeff.  \\
\midrule
Variable 1 OR Variable 4 & 2 \\
\midrule
Intercept & +1\\ 
\bottomrule
\end{tabular}
\vspace{1em}
\begin{tabular}{l| c}
\toprule
Learned Rules & Coeff.  \\
\midrule
Variable 1 OR Variable 4 & 1.63\\
\hline
Intercept &  +1.14\\ 
\bottomrule
\end{tabular}
\end{sc}
\end{small}
\end{center}
\vskip -0.1in
\label{app_tab:customized_descrptions_synth}
\end{table}

In Table~\ref{app_tab:customized_descrptions_synth}, we compare a set of learned rules (right Table) to the ground truth rules (left Table) from generated data. We interpret the results by saying that the model perfectly produces the correct rules, e.g. variable 1 and variable 4. Moreover, the coefficients and intercept are also identical if rounded. 

\clearpage
\section{Proof of Proposition 1}

\begin{thmprop*}
    Assume that an outcome $Y$ is linear in $X$ with noise of bounded conditional variance,
    $$
    Y = \beta^\top X + \epsilon(X), \mbox{ where }\; \E[\epsilon \mid X] = 0,  \V[\epsilon \mid X] \leq \sigma^2~,
    $$
    with $\beta \in \mathbb{R}^d$ and $X \in \{0,1\}^d$ a multivariate binary variable with the following structure. For each $X_i$ there is a paired ``replacement'' variable $X_{j(i)}$, with $j(j(i)) = i$, such that for $\delta \geq 0$, $p(X_i = X_{j(i)}) \geq 1-\delta$, and that whenever $X_i$ is missing,  $X_{j(i)}$ is observed, $M_i = 1 \Rightarrow M_{j(i)}=0$. Assume also that $\forall i, k \not\in \{i, j(i)\} : X_i \indep X_k$. Then, there is a GLRM $h$ with two-variable rules $\{\zX_i \lor \zX_{j(i)} \}_{i=1}^d$, where $\zX_i = (1-M_i)X_i$, with risk 
    $$
    R(h) \leq \delta \|\beta\|_2^2 + \delta^2 \sum_{i,k\not\in\{i, j(i)\}}|\beta_i \beta_k| + \sigma^2~.
    $$
    under the squared error. Additionally, if $\beta_i \geq 0$ and $\E[X_iM_i] \geq \eta$ for all $i \in [d]$, using the ground truth $\beta$ with zero-imputed features $\zX$ yields a risk bounded from below as 
    $$
    R(\beta) \geq \eta \|\beta\|_2^2 + \sigma^2~, 
    $$
    and a greater missingness reliance than the GLRM, $\bar{\rho}(\beta) \geq \bar{\rho}(h)$.
\end{thmprop*}

\begin{proof}
Let $\mu(X) = \E[Y \mid X]$. 
The risk of any hypothesis $h(X)$ can be decomposed as 
\begin{align*}
     R(h) = \E[L(h(X), Y)] = \E[(h(X) - Y)^2] = \\
\E[(h(X) - \mu(X))^2] + \underbrace{\E[\epsilon^2]}_{\leq \sigma^2}~.
\end{align*}

Now, consider a GRLM $h$ where each variable pair $i, j(i)$ is represented by a rule $(\zX_i \lor \zX_{j(i)})$, used in place of $X_i$ and $X_j$ in a linear model, and a coefficient $\tbeta_i = \beta_i + \beta_{j(i)}$. Then, for each $i$, define the bias variable
$$
\Delta_i = (\zX_i \lor \zX_{j(i)}) - X_i = \left\{
\begin{array}{ll}
1, & \mbox{ if } \zX_{j(i)} = 1 \land X_i = 0 \\
0, & \mbox{otherwise}
\end{array}
\right.~.
$$
In other words, bias is introduced, $\Delta_i=1$, only if the zero-imputed replacement $\zX_{j(i)}$ is 1 but $X_i$ is 0. $\zX_{j(i)}$ is only equal to 1 if $j(i)$ is observed. Thus, $\E[\Delta_i] = p(X_{j(i)}=1, X_i=0) \leq \delta$, by assumption. As a result, 
\begin{align*}
\E[(h(X) - \mu(X))^2] & = \E\left[\left( \sum_{i=1}^d (\beta_i (\zX_i \lor \zX_{j(i)}) - \beta_i X_i)\right)^2 \right] \\
& = \E\left[ \sum_{i,j=1}^d \beta_i \beta_j \Delta_i \Delta_j \right] \\
& = \sum_{i,j=1}^d \E[\beta_i \beta_j \Delta_i \Delta_j ] \\
& = \sum_{i=1}^d \left( \E[\beta_i^2 \Delta_i^2] + \E[\beta_i \beta_{j(i)} \underbrace{\Delta_i \Delta_{j(i)}}_{=0}] \right) + \\
& \sum_{k \not\in \{i,j(i)\}}\E[\beta_i \beta_k \Delta_i\Delta_k]  \\
& = \sum_{i=1}^d \left( \E[\beta_i^2 \Delta_i] + \sum_{k \not\in \{i,j(i)\}}\E[\beta_i \Delta_i]\E[\beta_k\Delta_k] \right) \\
& \leq \sum_{i=1}^d \left( \beta_i^2 \E[\Delta_i] + \sum_{k \not\in \{i,j(i)\}}\beta_i\beta_k \E[\Delta_i]\E[\Delta_k] \right) && \mbox{By independence, } X_i \indep X_k\\
& \leq \delta \|\beta\|^2 + \delta^2 \sum_{k \not\in \{i,j(i)\}}|\beta_i\beta_k |~.
\end{align*}
We can generalize the result by placing a bound on the cross-moment of the replacement bias $\E[\Delta_i \Delta_k]$, rather than assuming that $X_i \indep X_k$. 

There is also a lower bound for the ground-truth model applied to zero-imputed data with missingness. Its bias is
$$
B = \E[(\beta^\top X - \beta^\top \zX)^2] ) = \E[(\beta^\top (M \odot X))^2] 
$$
If all coefficiencts are positive, $\beta \in \bbR^d_+$,  and hence all terms in the bias, 
$$
B \geq \sum_{i=1}^d \E[(\beta_i M_i X_i)^2] = \sum_{i=1}^d \beta_i^2 \E[M_i X_i]
$$
By the assumption that $\E[M_i X_i] \geq \gamma$ for some $\gamma > 0$, it follows that 
$$
B \geq \gamma \|\beta\|_2^2~.
$$

The reliance on features with missing values $\bar{\rho}(h)$ of the GLRM $h$ is determined by events where a replacement variable $j(i)$ has the value $0$ when the variable $i$ is unobserved, $\exists i : \mathds{1}[M_i=1, X_{j(i)}=0)]$. If this is true for any $i$, $\rho=1$. For the ground-truth model, it is sufficient that a variable is missing, $\exists i: \mathds{1}[M_i=1]$. Hence, the expected reliance on features with missing values is greater for $\beta^\top \zX$ than for $h$.

In conclusion, the GLRM is preferred whenever 
$$
\delta \|\beta\|^2 + \delta^2 \sum_{i, k \not\in \{i,j(i)\}}|\beta_i\beta_k | < \gamma \|\beta\|^2~.
$$
Letting $a = \|\beta\|^2/(\sum_{i, k \not\in \{i,j(i)\}}|\beta_i\beta_k |)$ and solving for $\delta$, we get
$$
\delta < (\sqrt{a^2 + 4\eta} - a)/2 ~.
$$

\end{proof}

\vfill

\end{document}

%% file: header.tex
\DeclareMathOperator*{\argmin}{arg\,min}

\newtheorem{thmprop}{Proposition}
\newtheorem{thmprop*}{Proposition}

\newtheorem*{thmidasmp*}{Identifying assumptions}

\theoremstyle{definition}

\newtheorem{thmobs}{Observation}

\newtheorem*{thmrem*}{Remark}

\def\independenT#1#2{\mathrel{\rlap{$#1#2$}\mkern2mu{#1#2}}}
\def\indep{\independent{}}

\def\E{\mathbb{E}}
\def\V{\mathbb{V}}

\def\hbeta{\hat{\beta}}
\def\hcS{\hat{\mathcal{S}}}

\def\ba{{\bf a}}

\def\bA{{\bf A}}

\def\bM{{\bf M}}

\def\bR{{\bf R}}

\def\bX{{\bf X}}
\def\bY{{\bf Y}}

\def\bbR{\mathbb{R}}

\def\cK{\mathcal{K}}

\def\cO{\mathcal{O}}

\def\cS{\mathcal{S}}

\def\MINTY{{\tt MINTY}}
\def\LASSO{{\tt LASSO}}
\def\XGB{{\tt XGB}}
\def\NEUMISS{{\tt NEUMISS}}
\def\DT{{\tt DT}}

\def\Life{{\textit{Life}}}
\def\Housing{{\textit{Housing}}}
\def\ADNI{{\textit{ADNI}}}

\def\E{\mathbb{E}}

\def\cK{\mathcal{K}}

\def\bR{\mathbf{R}}
\def\bX{\mathbf{X}}
\def\bY{\mathbf{Y}}

\def\bM{\mathbf{M}}
\def\bA{\mathbf{A}}
\def\zx{\bar{x}}
\def\zX{\bar{X}}
\def\zbX{\bar{\mathbf{X}}}
\def\tbeta{\tilde{\beta}}

\def\na{{\tt NA}}

\def\olcA
\def\olcW{\overline{\mathcal{W}}}

\def\ds1{\mathds{1}}

\def\cO{\mathcal{O}}

\def\bbR{\mathbb{R}}
\def\na{{\tt NA}}

%% file: main.bbl
\begin{thebibliography}{39}
\providecommand{\natexlab}[1]{#1}
\providecommand{\url}[1]{\texttt{#1}}
\expandafter\ifx\csname urlstyle\endcsname\relax
  \providecommand{\doi}[1]{doi: #1}\else
  \providecommand{\doi}{doi: \begingroup \urlstyle{rm}\Url}\fi

\bibitem[Afessa et~al.(2005)Afessa, Keegan, Gajic, Hubmayr, and Peters]{afessa2005influence}
Bekele Afessa, Mark~T Keegan, Ognjen Gajic, Rolf~D Hubmayr, and Steve~G Peters.
\newblock The influence of missing components of the acute physiology score of apache iii on the measurement of icu performance.
\newblock \emph{Intensive care medicine}, 31:\penalty0 1537--1543, 2005.

\bibitem[Bengio and Gingras(1995)]{bengio1995recurrent}
Yoshua Bengio and Francois Gingras.
\newblock Recurrent neural networks for missing or asynchronous data.
\newblock \emph{Advances in neural information processing systems}, 8, 1995.

\bibitem[Buitinck et~al.(2013)Buitinck, Louppe, Blondel, Pedregosa, Mueller, Grisel, Niculae, Prettenhofer, Gramfort, Grobler, Layton, VanderPlas, Joly, Holt, and Varoquaux]{sklearn_api}
Lars Buitinck, Gilles Louppe, Mathieu Blondel, Fabian Pedregosa, Andreas Mueller, Olivier Grisel, Vlad Niculae, Peter Prettenhofer, Alexandre Gramfort, Jaques Grobler, Robert Layton, Jake VanderPlas, Arnaud Joly, Brian Holt, and Ga{\"{e}}l Varoquaux.
\newblock {API} design for machine learning software: experiences from the scikit-learn project.
\newblock In \emph{ECML PKDD Workshop: Languages for Data Mining and Machine Learning}, pages 108--122, 2013.

\bibitem[Carpenter and Kenward(2012)]{carpenter2012multiple}
James Carpenter and Michael Kenward.
\newblock \emph{{Multiple imputation and its application}}.
\newblock John Wiley \& Sons, 2012.

\bibitem[Che et~al.(2018)Che, Purushotham, Cho, Sontag, and Liu]{che2018recurrent}
Zhengping Che, Sanjay Purushotham, Kyunghyun Cho, David Sontag, and Yan Liu.
\newblock Recurrent neural networks for multivariate time series with missing values.
\newblock \emph{Scientific reports}, 8\penalty0 (1):\penalty0 1--12, 2018.

\bibitem[Chen and Guestrin(2016)]{chen2016xgboost}
Tianqi Chen and Carlos Guestrin.
\newblock Xgboost: A scalable tree boosting system.
\newblock In \emph{Proceedings of the 22nd acm sigkdd international conference on knowledge discovery and data mining}, pages 785--794, 2016.

\bibitem[Chen et~al.(2019)Chen, He, Benesty, and Khotilovich]{chen2019package}
Tianqi Chen, Tong He, Michael Benesty, and Vadim Khotilovich.
\newblock {Package ‘xgboost’}.
\newblock \emph{R version}, 90, 2019.

\bibitem[Chen et~al.(2023)Chen, Tan, Chajewska, Rudin, and Caruna]{chen2023missing}
Zhi Chen, Sarah Tan, Urszula Chajewska, Cynthia Rudin, and Rich Caruna.
\newblock Missing values and imputation in healthcare data: Can interpretable machine learning help?
\newblock In \emph{Conference on Health, Inference, and Learning}, pages 86--99. PMLR, 2023.

\bibitem[De~Cock(2011)]{de2011ames}
Dean De~Cock.
\newblock Ames, iowa: Alternative to the boston housing data as an end of semester regression project.
\newblock \emph{Journal of Statistics Education}, 19\penalty0 (3), 2011.

\bibitem[F{\"u}rnkranz et~al.(2012)F{\"u}rnkranz, Gamberger, and Lavra{\v{c}}]{furnkranz2012foundations}
Johannes F{\"u}rnkranz, Dragan Gamberger, and Nada Lavra{\v{c}}.
\newblock \emph{Foundations of rule learning}.
\newblock Springer Science \& Business Media, 2012.

\bibitem[{Gurobi Optimization, LLC}(2023)]{gurobi}
{Gurobi Optimization, LLC}.
\newblock {Gurobi Optimizer Reference Manual}, 2023.
\newblock URL \url{https://www.gurobi.com}.

\bibitem[Haniffa et~al.(2018)Haniffa, Isaam, De~Silva, Dondorp, and De~Keizer]{haniffa2018performance}
Rashan Haniffa, Ilhaam Isaam, A~Pubudu De~Silva, Arjen~M Dondorp, and Nicolette~F De~Keizer.
\newblock Performance of critical care prognostic scoring systems in low and middle-income countries: a systematic review.
\newblock \emph{Critical care}, 22:\penalty0 1--22, 2018.

\bibitem[Ibrahim et~al.(2005)Ibrahim, Chen, Lipsitz, and Herring]{ibrahim2005missing}
Joseph~G Ibrahim, Ming-Hui Chen, Stuart~R Lipsitz, and Amy~H Herring.
\newblock Missing-data methods for generalized linear models: A comparative review.
\newblock \emph{Journal of the American Statistical Association}, 100\penalty0 (469):\penalty0 332--346, 2005.

\bibitem[Jones(1996)]{jones1996}
Michael~P. Jones.
\newblock Indicator and stratification methods for missing explanatory variables in multiple linear regression.
\newblock \emph{Journal of the American Statistical Association}, 91\penalty0 (433):\penalty0 222--230, 1996.
\newblock ISSN 01621459.
\newblock URL \url{http://www.jstor.org/stable/2291399}.

\bibitem[Josse et~al.(2019)Josse, Prost, Scornet, and Varoquaux]{josse2019consistency}
Julie Josse, Nicolas Prost, Erwan Scornet, and Ga{\"e}l Varoquaux.
\newblock On the consistency of supervised learning with missing values.
\newblock \emph{arXiv preprint arXiv:1902.06931}, 2019.

\bibitem[Le~Morvan et~al.(2020{\natexlab{a}})Le~Morvan, Josse, Moreau, Scornet, and Varoquaux]{le2020neumiss}
Marine Le~Morvan, Julie Josse, Thomas Moreau, Erwan Scornet, and Ga{\"e}l Varoquaux.
\newblock Neumiss networks: differentiable programming for supervised learning with missing values.
\newblock \emph{Advances in Neural Information Processing Systems}, 33:\penalty0 5980--5990, 2020{\natexlab{a}}.

\bibitem[Le~Morvan et~al.(2020{\natexlab{b}})Le~Morvan, Josse, Moreau, Scornet, and Varoquaux]{morvan2020neumiss}
Marine Le~Morvan, Julie Josse, Thomas Moreau, Erwan Scornet, and Ga{\"e}l Varoquaux.
\newblock {NeuMiss networks: differentiable programming for supervised learning with missing values}.
\newblock \emph{arXiv:2007.01627}, 2020{\natexlab{b}}.

\bibitem[Le~Morvan et~al.(2020{\natexlab{c}})Le~Morvan, Prost, Josse, Scornet, and Varoquaux]{lemorvan20a_linear}
Marine Le~Morvan, Nicolas Prost, Julie Josse, Erwan Scornet, and Gael Varoquaux.
\newblock Linear predictor on linearly-generated data with missing values: non consistency and solutions.
\newblock In Silvia Chiappa and Roberto Calandra, editors, \emph{Proceedings of the Twenty Third International Conference on Artificial Intelligence and Statistics}, volume 108 of \emph{Proceedings of Machine Learning Research}, pages 3165--3174. PMLR, 26--28 Aug 2020{\natexlab{c}}.

\bibitem[Le~Morvan et~al.(2021)Le~Morvan, Josse, Scornet, and Varoquaux]{le2021sa}
Marine Le~Morvan, Julie Josse, Erwan Scornet, and Ga{\"e}l Varoquaux.
\newblock What’s a good imputation to predict with missing values?
\newblock \emph{Advances in Neural Information Processing Systems}, 34:\penalty0 11530--11540, 2021.

\bibitem[Little and Rubin(2019)]{little2019statistical}
Roderick~JA Little and Donald~B Rubin.
\newblock \emph{Statistical analysis with missing data}, volume 793.
\newblock John Wiley \& Sons, 2019.

\bibitem[Margot and Luta(2021)]{margot2021new}
Vincent Margot and George Luta.
\newblock A new method to compare the interpretability of rule-based algorithms.
\newblock \emph{AI}, 2\penalty0 (4):\penalty0 621--635, 2021.

\bibitem[Mayer et~al.(2019)Mayer, Sportisse, Josse, Tierney, and Vialaneix]{Mayer2019}
Imke Mayer, Aude Sportisse, Julie Josse, Nicholas Tierney, and Nathalie Vialaneix.
\newblock R-miss-tastic: a unified platform for missing values methods and workflows, 2019.

\bibitem[Mofrad et~al.(2021)Mofrad, Lundervold, Vik, and Lundervold]{mofrad2021cognitive}
Samaneh~A Mofrad, Astri~J Lundervold, Alexandra Vik, and Alexander~S Lundervold.
\newblock {Cognitive and MRI trajectories for prediction of Alzheimer’s disease}.
\newblock \emph{Scientific Reports}, 11\penalty0 (1):\penalty0 1--10, 2021.

\bibitem[Nazabal et~al.(2020)Nazabal, Olmos, Ghahramani, and Valera]{nazabal2020handling}
Alfredo Nazabal, Pablo~M Olmos, Zoubin Ghahramani, and Isabel Valera.
\newblock {Handling incomplete heterogeneous data using vaes}.
\newblock \emph{Pattern Recognition}, 107:\penalty0 107501, 2020.

\bibitem[Oberst et~al.(2020)Oberst, Johansson, Wei, Gao, Brat, Sontag, and Varshney]{oberst2020characterization}
Michael Oberst, Fredrik Johansson, Dennis Wei, Tian Gao, Gabriel Brat, David Sontag, and Kush Varshney.
\newblock Characterization of overlap in observational studies.
\newblock In \emph{International Conference on Artificial Intelligence and Statistics}, pages 788--798. PMLR, 2020.

\bibitem[Organization et~al.(2021)]{world2021ghe}
World~Health Organization et~al.
\newblock Ghe: Life expectancy and healthy life expectancy.
\newblock \emph{The Global Health Observatory [Internet].[cited 26 Aug 2022]. https://www. who. int/data/gho/data/themes/mortality-andglobal-health-estimates/ghe-life-expectancy-and-healthy-life-expectancy}, 2021.

\bibitem[Pedregosa et~al.(2011)Pedregosa, Varoquaux, Gramfort, Michel, Thirion, Grisel, Blondel, Prettenhofer, Weiss, Dubourg, Vanderplas, Passos, Cournapeau, Brucher, Perrot, and Duchesnay]{scikit-learn_imp}
F.~Pedregosa, G.~Varoquaux, A.~Gramfort, V.~Michel, B.~Thirion, O.~Grisel, M.~Blondel, P.~Prettenhofer, R.~Weiss, V.~Dubourg, J.~Vanderplas, A.~Passos, D.~Cournapeau, M.~Brucher, M.~Perrot, and E.~Duchesnay.
\newblock {Scikit-learn: Machine Learning in {P}ython}.
\newblock \emph{Journal of Machine Learning Research}, 12:\penalty0 2825--2830, 2011.

\bibitem[Roser et~al.(2013)Roser, Ortiz-Ospina, and Ritchie]{owidlifeexpectancy}
Max Roser, Esteban Ortiz-Ospina, and Hannah Ritchie.
\newblock Life expectancy.
\newblock \emph{Our World in Data}, 2013.
\newblock https://ourworldindata.org/life-expectancy.

\bibitem[Rubin(1976)]{rubin1976inference}
Donald~B Rubin.
\newblock Inference and missing data.
\newblock \emph{Biometrika}, 63\penalty0 (3):\penalty0 581--592, 1976.

\bibitem[Rubin(1988)]{rubin1988overview}
Donald~B Rubin.
\newblock An overview of multiple imputation.
\newblock In \emph{Proceedings of the survey research methods section of the American statistical association}, volume~79, page~84. Citeseer, 1988.

\bibitem[Rucker et~al.(2015)Rucker, McShane, and Preacher]{rucker2015researcher}
Derek~D Rucker, Blakeley~B McShane, and Kristopher~J Preacher.
\newblock A researcher's guide to regression, discretization, and median splits of continuous variables.
\newblock \emph{Journal of Consumer Psychology}, 25\penalty0 (4):\penalty0 666--678, 2015.

\bibitem[Seaman et~al.(2013)Seaman, Galati, Jackson, and Carlin]{seaman2013meant}
Shaun Seaman, John Galati, Dan Jackson, and John Carlin.
\newblock What is meant by "missing at random"?
\newblock \emph{Statistical Science}, 28\penalty0 (2):\penalty0 257--268, 2013.

\bibitem[Stempfle et~al.(2023)Stempfle, Panahi, and Johansson]{stempfle2023sharing}
Lena Stempfle, Ashkan Panahi, and Fredrik~D Johansson.
\newblock Sharing pattern submodels for prediction with missing values.
\newblock In \emph{Proceedings of the AAAI Conference on Artificial Intelligence}, volume 37--8, pages 9882--9890, 2023.

\bibitem[Twala et~al.(2008)Twala, Jones, and Hand]{twala2008good}
Bheki~ETH Twala, MC~Jones, and David~J Hand.
\newblock Good methods for coping with missing data in decision trees.
\newblock \emph{Pattern Recognition Letters}, 29\penalty0 (7):\penalty0 950--956, 2008.

\bibitem[Ustun and Rudin(2019)]{ustun2019learning}
Berk Ustun and Cynthia Rudin.
\newblock Learning optimized risk scores.
\newblock \emph{Journal of Machine Learning Research (JMLR)}, 20\penalty0 (150):\penalty0 1--75, 2019.

\bibitem[Van~Buuren(2018)]{buren}
Stef Van~Buuren.
\newblock \emph{{Flexible Imputation of Missing Data (2nd ed.)}}.
\newblock Chapman and Hall/CRC, Boca Raton, FL, 2018.

\bibitem[Webb et~al.(2010)Webb, Keogh, and Miikkulainen]{webb2010naive}
Geoffrey~I Webb, Eamonn Keogh, and Risto Miikkulainen.
\newblock Na{\"\i}ve bayes.
\newblock \emph{Encyclopedia of machine learning}, 15\penalty0 (1):\penalty0 713--714, 2010.

\bibitem[Wei et~al.(2019)Wei, Dash, Gao, and Gunluk]{wei2019generalized}
Dennis Wei, Sanjeeb Dash, Tian Gao, and Oktay Gunluk.
\newblock Generalized linear rule models.
\newblock In \emph{International Conference on Machine Learning}, pages 6687--6696. PMLR, 2019.

\bibitem[Weiner et~al.(2010)Weiner, Aisen, Jack~Jr., Jagust, Trojanowski, Shaw, Saykin, Morris, Cairns, Beckett, Toga, Green, Walter, Soares, Snyder, Siemers, Potter, Cole, Schmidt, and Initiative]{ADNI}
Michael~W. Weiner, Paul~S. Aisen, Clifford~R. Jack~Jr., William~J. Jagust, John~Q. Trojanowski, Leslie Shaw, Andrew~J. Saykin, John~C. Morris, Nigel Cairns, Laurel~A. Beckett, Arthur Toga, Robert Green, Sarah Walter, Holly Soares, Peter Snyder, Eric Siemers, William Potter, Patricia~E. Cole, Mark Schmidt, and Alzheimer's Disease~Neuroimaging Initiative.
\newblock The alzheimer's disease neuroimaging initiative: Progress report and future plans.
\newblock \emph{Alzheimer's \& Dementia}, 6\penalty0 (3):\penalty0 202--211.e7, 2010.

\end{thebibliography}
